\documentclass{article}
\usepackage{ijcai23}

\usepackage{times}
\usepackage{soul}
\usepackage{url}
\usepackage[hidelinks]{hyperref}
\usepackage[utf8]{inputenc}
\usepackage[small]{caption}
\usepackage{graphicx}
\usepackage{amsmath}
\usepackage{amsthm}
\usepackage{booktabs}
\usepackage{algorithmic}
\usepackage[switch]{lineno}

\usepackage{amssymb}

\usepackage[ruled,vlined]{algorithm2e}

\usepackage{quoting}
\usepackage{subfigure}
\usepackage{multirow}

\newtheorem{theorem}{Theorem}
\newtheorem*{theorem*}{Theorem}

\title{VBMO: Voting-Based Multi-Objective Path Planning}

\author{
    Raj Korpan
    \affiliations
    Iona University
    \emails
    rkorpan@iona.edu
}

\begin{document}

\maketitle

\begin{abstract}
This paper presents VBMO, the Voting-Based Multi-Objective path planning algorithm, that generates optimal single-objective plans, evaluates each of them with respect to the other objectives, and selects one with a voting mechanism. VBMO does not use hand-tuned weights, consider the multiple objectives at every step of search, or use an evolutionary algorithm. Instead, it considers how a plan that is optimal in one objective may perform well with respect to others. VBMO incorporates three voting mechanisms: range, Borda, and combined approval. Extensive evaluation in diverse and complex environments demonstrates the algorithm's ability to efficiently produce plans that satisfy multiple objectives.
\end{abstract}

\section{Introduction}
Autonomous agents must often simultaneously address multiple competing objectives, such as distance, time, and safety, when path planning. Faced with multiple objectives, previous search-based approaches have either compromised among those objectives at each step in the search process or relied on a mathematical combination of them hand-tuned for a particular environment. The thesis of this work is that voting provides a more efficient way to find consensus among multiple objectives. This paper describes \textit{VBMO}, a voting-based multi-objective path planning approach, where a set of single-objective planners each constructs its own optimal plan with A* and then uses a voting mechanism to select the plan that performs best across all objectives. The algorithm is evaluated in many challenging environments.

An optimal graph-search algorithm finds the least cost path from a start vertex to a target vertex. Typically, the algorithm exploits a weighted graph that represents a real-world problem, such as a navigable two-dimensional space. Such a graph $G = (V,E)$ represents unobstructed locations there as vertices $V$. Edges $E$ in $G$ each connect two vertices, normally only if one can move directly between them. Each edge is associated with a label for the cost to traverse it. For example, if the objective were to minimize path length, edge labels could record the Euclidean distance between pairs of vertices. Without loss of generality, optimization here is assumed to be search for a minimum cost.

A single-objective path planner $SO$ seeks a plan $P$ in $G$ that minimizes a single objective $\beta$, such as distance. A plan $P$ is \textit{optimal} with respect to $\beta$ only if no other plan $P'$ has a lower total cost $\beta(P)$ for that objective, that is, for every other plan $P', \beta(P) \leq \beta(P')$. 

A \textit{multi-objective} path planner $MO$ seeks a plan $P$ that performs well with respect to a set $B$ of objectives. It can guide search in that graph with hand-tuned weights that encapsulate its objectives' costs into a single value or it construct a new compromise among its objectives at every plan step. 
If, for example, $B = \{\beta_1, \beta_2\}$, where $\beta_1$ is travel distance and $\beta_2$ is proximity to obstacles, $MO$ would seek a plan $P$ that scores well on both objectives. Because objectives may conflict, no single plan is likely to be optimal with respect to all of $B$. Typically, a potential plan will perform better with respect to some $\beta$'s and worse with respect to others. 

Let $B=\{\beta_1, \beta_2,\ldots,\beta_J\}$ be a set of $J$ planning objectives with respective costs for plan $P$ as $\{\beta_1(P),\ldots,\beta_J(P)\}$ calculated in a graph labeled with their individual objectives. A plan $P_1$ \textit{dominates} another plan $P_2$ ($P_1 \ll P_2$) when $\beta(P_1) \leq \beta(P_2)$ for every $\beta \in B$ and $\beta_j(P_1) < \beta_j(P_2)$ for at least one objective $\beta_j \in B$. Dominance is transitive, that is, if $P_1 \ll P_2$ and $P_2 \ll P_3$, then $P_1 \ll P_3$ \cite{pardalos2008pareto}. Among all possible plans, a non-dominated plan lies on the \textit{Pareto frontier}, the set of all solutions that cannot be improved on one objective without a penalty to another objective \cite{lavalle2006planning}. A typical multi-objective planner searches for plans that lie on the Pareto frontier and then an external decision maker chooses among them. A plan is \textit{Pareto optimal} if and only if no other plan dominates it.

We present an algorithm for multi-objective path planning, called \textit{VBMO}, that produces a set of plans efficiently and then uses a voting mechanism to select a plan that is optimal for at least one objective and is guaranteed to be on the Pareto frontier. Unlike other approaches, VBMO does not need to modify the operations of the search algorithm. Instead, it relies on the efficiency of single-objective search and uses post-hoc evaluation and voting to select a non-dominated plan.

The next sections provide related work in multi-objective path planning and describe VBMO. The final sections present empirical results and discuss future work.

\begin{figure*}[t]
\centering
\subfigure[Map lak110d from Dragon Age: Origins]{
\includegraphics[width = 0.3\linewidth]{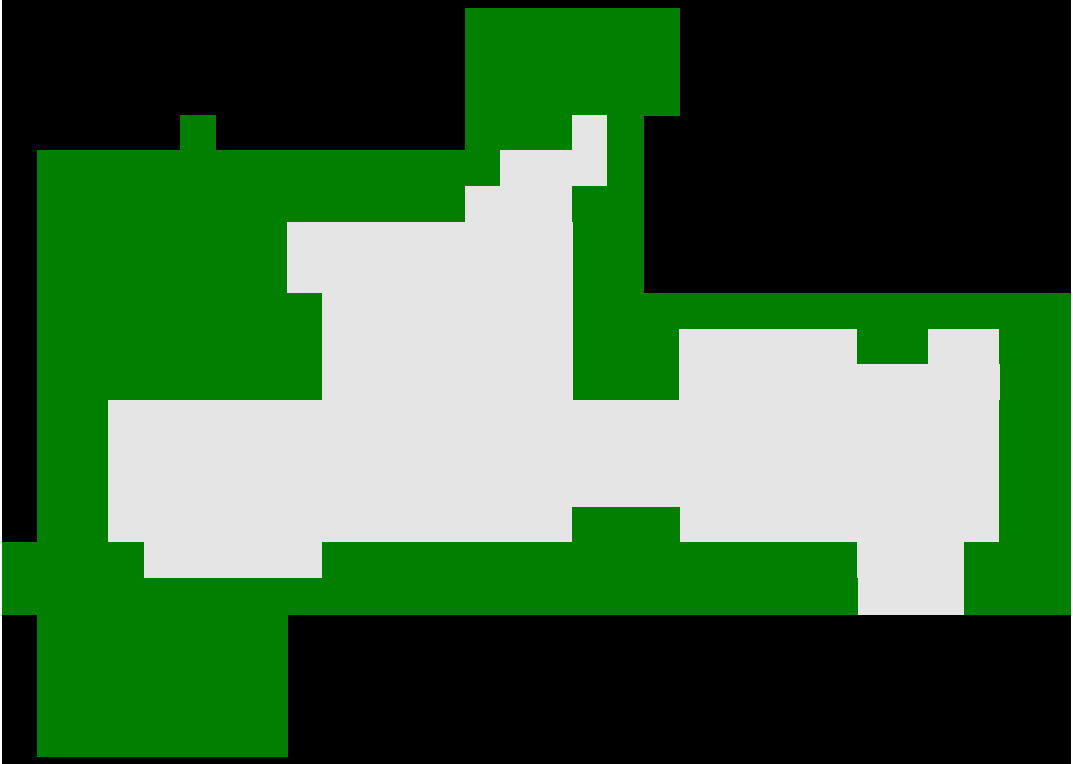}}
\subfigure[Graph of environment]{
\includegraphics[width = 0.49\linewidth]{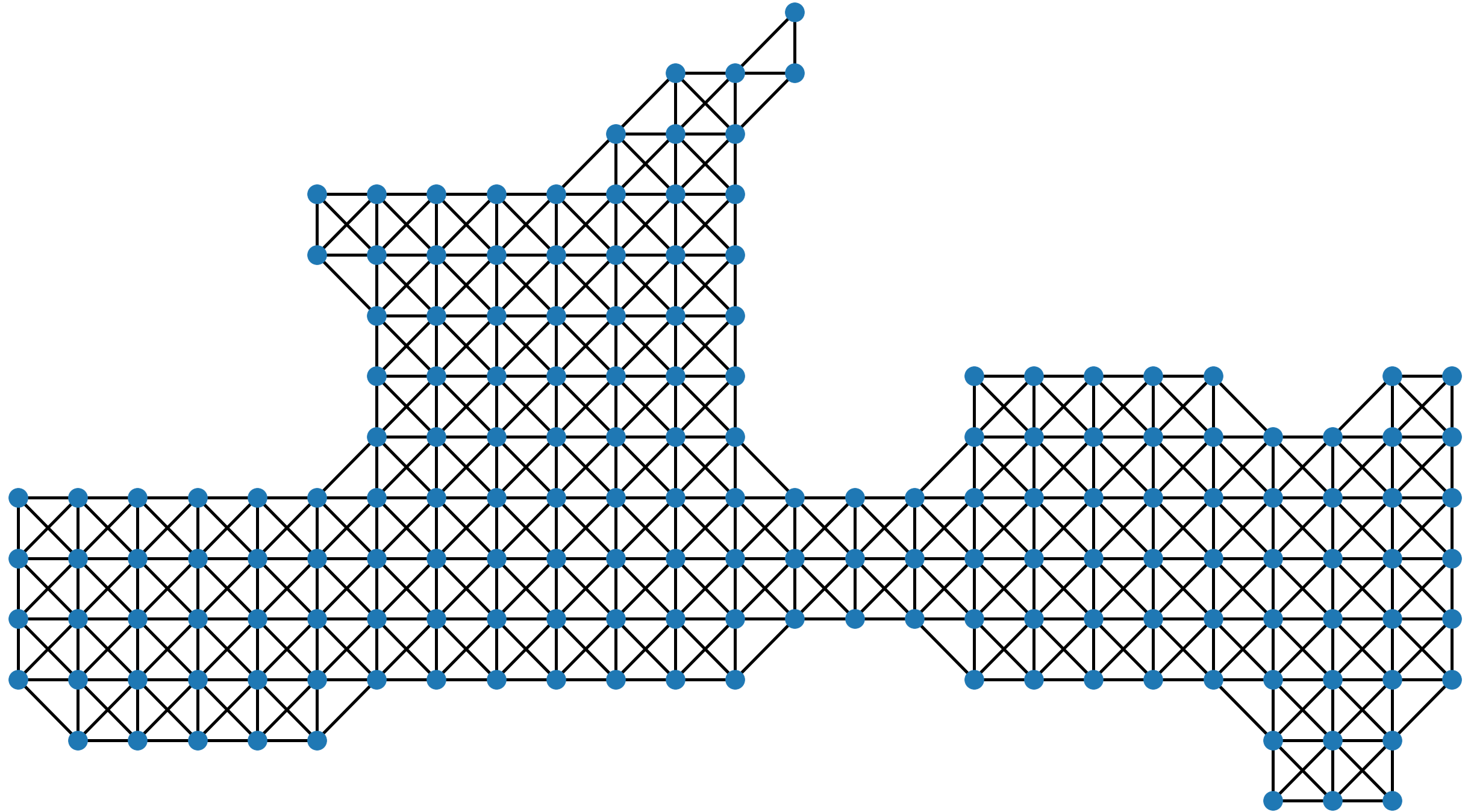}}
\caption{Each environment is converted to a graph whose topology is shared across all objectives.}
\label{fig:environment}
\end{figure*}

\begin{figure*}[t]
\centering
\subfigure[Distance]{
\includegraphics[width = 0.33\linewidth]{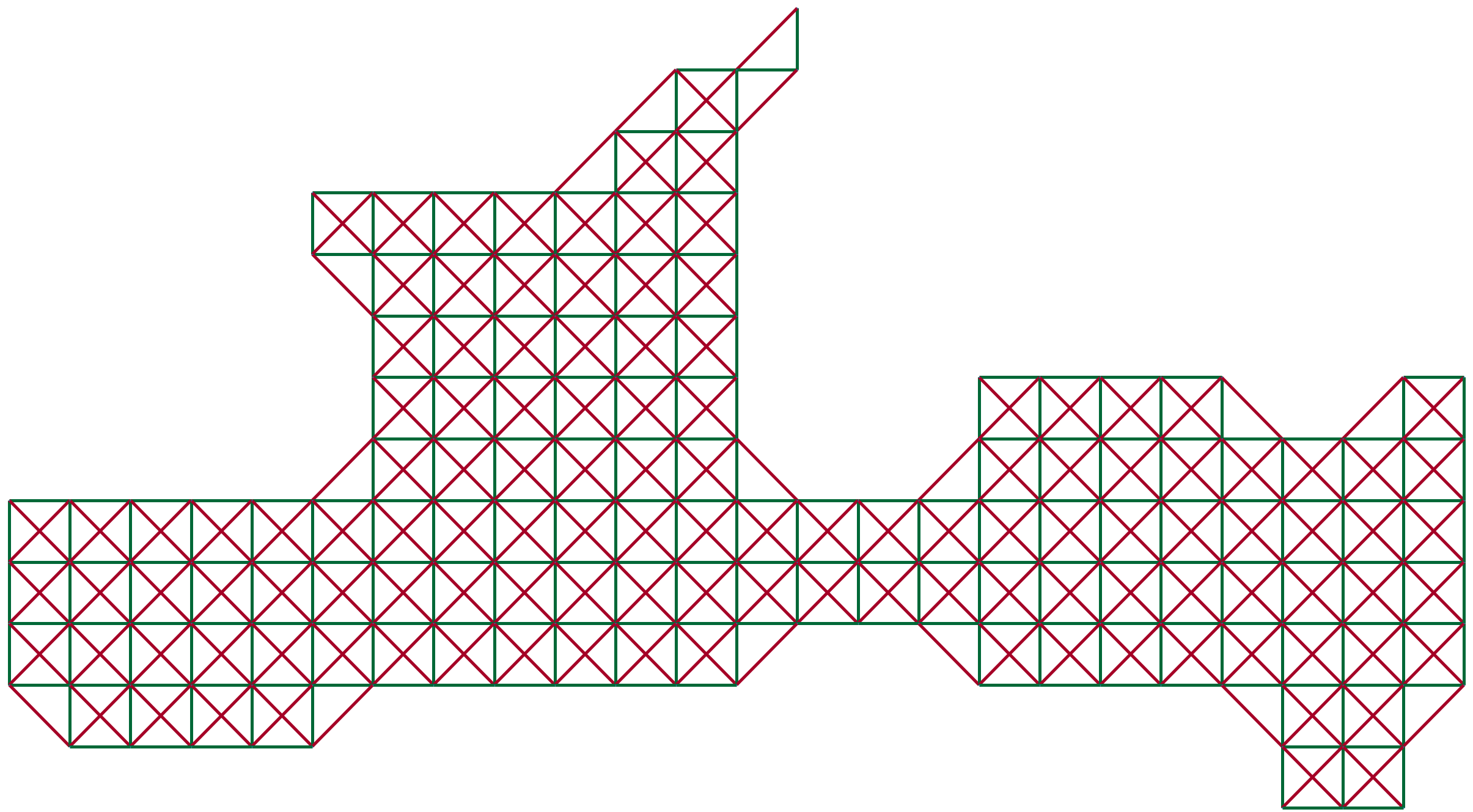}}
\subfigure[Safety]{
\includegraphics[width = 0.33\linewidth]{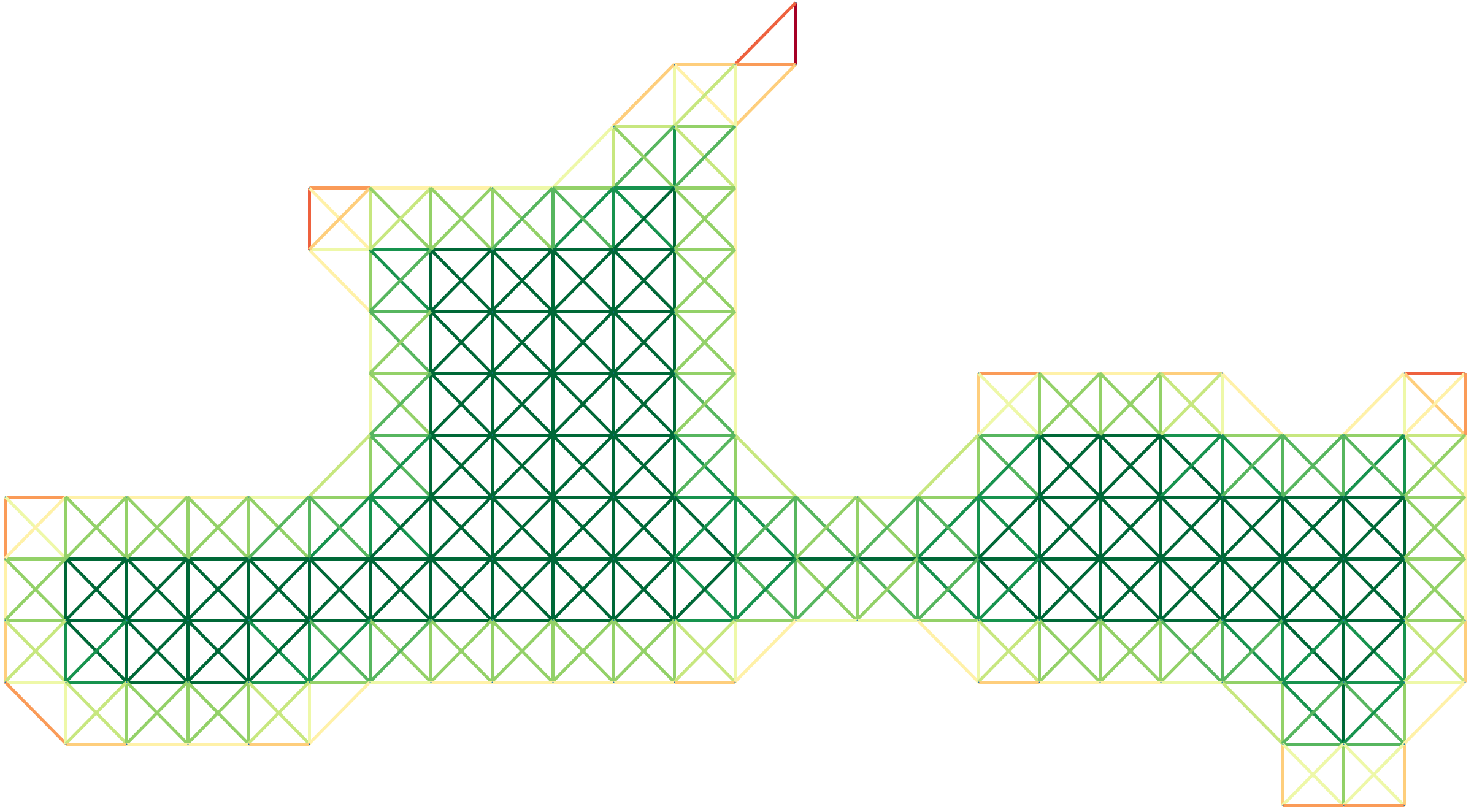}}
\subfigure[Random]{
\includegraphics[width = 0.33\linewidth]{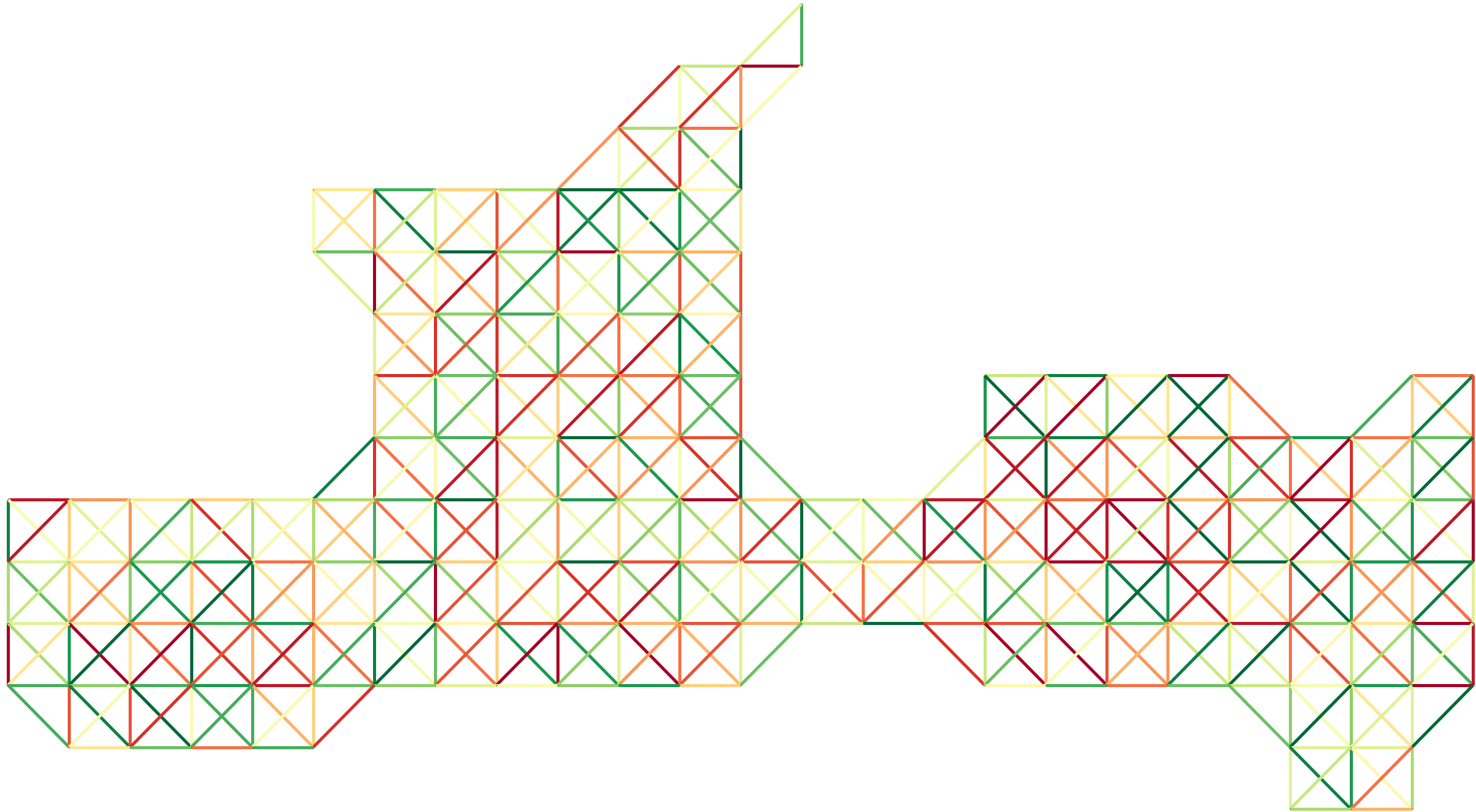}}
\caption{Each objective labels edges in the graph differently. Edges are colored in a range from green for low cost to red for high cost. (a) Distance is measured as 1 for vertical and horizontal edges and 1.414 for diagonals. (b) Safety is measured as the maximum degree in the graph (8 in this example) plus 1 minus the average degree of the two vertices joined by the edge. (c) Random costs are generated uniformly from 1 to 20.}
\label{fig:costs}
\end{figure*}

\section{Related Work}
An early approach to multi-objective optimization treated it as a single-objective problem for a simple weighted sum of the objectives \cite{zadeh1963optimality}. Others addressed individual objectives in a weighted sum with constraints \cite{haimes1973integrated}, minimum values \cite{lee1972goal}, or ideal values \cite{wierzbicki1980use}. The weighted sum approach has also been applied to the heuristic function of an optimal search algorithm \cite{refanidis2003multiobjective}. All this work, however, required a human expert with knowledge of the relative importance of the objectives to tune the weights \cite{marler2010weighted}. Moreover, small changes in those weights can result in dramatically different plans.

Many have used metaheuristics (e.g., evolutionary algorithms) to find non-dominated solutions to multi-objective problems \cite{deb2002fast}. Those approaches, however, do not guarantee optimality, require tuning many hyperparameters, and are computationally expensive \cite{talbi2012multi}. Furthermore, as the number of objectives increases, the fraction of non-dominated solutions among all solutions approaches one \cite{farina2002optimal} and the size of the Pareto frontier increases exponentially \cite{jaimes2015many}. As a result, methods that seek Pareto dominance break down with more objectives because it becomes more computationally expensive to compare all the potential non-dominated solutions. VBMO avoids this computation on infinitely many points on the surface of the Pareto frontier. Instead, it only ever compares $|B|$ solutions because it transforms the multi-objective problem into a set of single-objective problems.

A*, the traditional optimal search algorithm, requires an admissible heuristic, one that consistently underestimates its objective \cite{hart1968formal}. Several approaches extend A* to address multi-objective search. Multi-objective A* tracks all the objectives simultaneously as it maintains a queue of search nodes to expand \cite{stewart1991multiobjective}. NAMOA* extends multi-objective A* with a queue of partial solution paths instead of search nodes, but it is slow, memory hungry, and does not scale well \cite{mandow2008multiobjective}. Multi-heuristic A* modifies A* to consider multiple heuristics, some of which can be inadmissible \cite{aine2016multi}. It interleaves expansion of search nodes selected by an admissible heuristic with expansion on search nodes selected by nonadmissible ones. This approach was extended to treat the expansion from nonadmissible heuristics as a multi-armed bandit problem \cite{phillips2015efficient}. Recent work has sought to identify informative admissible heuristics for an improved version of NAMOA* \cite{geisser2022admissible}.

Other work has addressed these issues of efficiency and scale but only for two objectives \cite{ulloa2020simple,goldin2021approximate}. Also in the bi-objective context, others have focused on finding the extreme supported non-dominated plans on the Pareto frontier \cite{sedeno2015dijkstra} or approximate Pareto-optimal solutions \cite{zhang2022pex}. VBMO also produces non-dominated plans on the Pareto frontier but for multiple objectives. Others have sought to extend multi-objective path planning to the multi-agent context with subdimensional expansion \cite{ren2021subdimensional} and conflict-based search \cite{ren2021multi}.

\begin{algorithm}[!t]
    \DontPrintSemicolon
    \textbf{Input:} \textit{objectives $B$, map of environment, voting mechanism}\;
    Construct graph $G$ from map of environment\;
    $\mathcal{P} \gets \{ \}$\;
    \For{each objective $\beta_j \in B$}
    {
        Add $\beta_j$ to each edge's cost vector \;
        A* search to find plan $P_j$ in $G$ that minimizes $\beta_j$\;
        $\mathcal{P} \gets \mathcal{P} \cup P_j$\;
    }
    \For{each objective $\beta_j \in B$}
    {
        \For{each plan $P_i \in \mathcal{P}$}
        {
            $C_{ij} \gets$ cost of plan $P_i$ from $\beta_j$\;
        }
        Normalize plan scores $C_{ij}$ in [0,1]\;
    }
    \Switch{voting mechanism}{
        \Case{range voting}{
            \For{each plan $P_i \in \mathcal{P}$}
            {
                $Score_i \gets \sum_{j=1}^J C_{ij}$\;
            }
            $best \gets argmin_i \ Score_i$\;
        }
        \Case{Borda voting}{
            \For{each plan $P_i \in \mathcal{P}$}
            {
                $r_{ij} \gets$ rank of $P_i$\;
                $points_{ij} \gets (J+1)-r_{ij}$\;
            }
            \For{each plan $P_i \in \mathcal{P}$}
            {
                $Points_i \gets \sum_{j=1}^J points_{ij}$\;
            }
            $best \gets argmax_i \ Points_i$\;
        }
        \Case{combined approval voting}{
            \For{each plan $P_i \in \mathcal{P}$}
            {
                \uIf{$C_{ij} = 1$}{
                    $v_{ij} \gets$ -1\;
                }
                \uElseIf{$C_{ij} = 0$}{
                    $v_{ij} \gets$ 1\;
                }
                \uElse {
                    $v_{ij} \gets$ 0\;
                }
            }
            \For{each plan $P_i \in \mathcal{P}$}
            {
                $Value_i \gets \sum_{j=1}^J v_{ij}$\;
            }
            $best \gets argmax_i \ Value_i$\;
        }
    }
\Return{$P_{best}$}
\caption{VBMO planning algorithm}
\label{alg:planning}
\end{algorithm}

\begin{table}[t]
\centering
\caption{Each voting mechanism considers the scores $C_{ij}$ for six plans $P_i$ given six objectives $\beta_j$. Normalization ensures that each plan has minimum score with respect to its own objective. Range voting selects the plan with minimum total score, here $P_2$, Borda voting selects the plan with maximum total points, here $P_3$, and combined approval voting selects the plan with maximum total value, here $P_1$. Although all three methods start with the same set of scores, the difference between the voting mechanisms results in a different plan being selected as $P_{best}$.}
\label{tab:votingexample}
\begin{tabular}{|l|cccccc|r|}
\hline
Range             & $\beta_1$ & $\beta_2$ & $\beta_3$ & $\beta_4$ & $\beta_5$ & $\beta_6$ & \multicolumn{1}{c|}{Total Score}  \\ \hline
$P_1$ & 0   & 0.1  & 0.5   & 0.6  & 0    & 0.7  & 1.9 \\
$P_2$ & 0.1 & 0    & 0.2   & 0.2  & 1    & 0.3  & \textbf{1.8} \\
$P_3$ & 0.2 & 0.7  & 0     & 0.1  & 0.7  & 0.2  & 1.9 \\
$P_4$ & 1   & 0.8  & 1     & 0    & 0.3  & 1    & 4.1 \\
$P_5$ & 0.5 & 1    & 0.2   & 1    & 0    & 0.6  & 3.3   \\
$P_6$ & 0.5 & 1    & 0.2   & 0.1  & 1    & 0    & 2.8 \\ \hline
Borda             & $\beta_1$ & $\beta_2$ & $\beta_3$ & $\beta_4$ & $\beta_5$ & $\beta_6$ & \multicolumn{1}{c|}{Total Points} \\ \hline
$P_1$ & 6   & 5   & 4   & 3   & 6   & 2   & 26   \\
$P_2$ & 5   & 6   & 5   & 4   & 3   & 4   & 27   \\
$P_3$ & 4   & 4   & 6   & 5   & 4   & 5   & \textbf{28}   \\
$P_4$ & 2   & 3   & 3   & 6   & 5   & 1   & 20   \\
$P_5$ & 3   & 2   & 5   & 2   & 6   & 3   & 21   \\
$P_6$ & 3   & 2   & 5   & 5   & 3   & 6   & 24   \\ \hline
CAV & $\beta_1$ & $\beta_2$ & $\beta_3$ & $\beta_4$ & $\beta_5$ & $\beta_6$ & \multicolumn{1}{c|}{Total Value}  \\ \hline
$P_1$ & 1   & 0   & 0   & 0   & 1   & 0   & \textbf{2}    \\
$P_2$ & 0   & 1   & 0   & 0   & -1  & 0   & 0    \\
$P_3$ & 0   & 0   & 1   & 0   & 0   & 0   & 1    \\
$P_4$ & -1  & 0   & -1  & 1   & 0   & -1  & -2   \\
$P_5$ & 0   & -1  & 0   & -1  & 1   & 0   & -1   \\
$P_6$ & 0   & -1  & 0   & 0   & -1  & 1   & -1   \\ \hline
\end{tabular}
\end{table}

Some multi-objective approaches draw from social choice theory and voting systems. For example, in multi-attribute utility theory a function evaluates the available choices and selects the one with greatest utility \cite{keeney1993decisions}. Given a set of voters that cast a number of votes with respect to a set of \textit{candidates} (i.e., choices), a \textit{voting method} selects the winning candidate \cite{van2002overview}. The goal of a voting method is to weigh the voters' choices to select a winning candidate that fairly balances all the voters' opinions. Voting methods have incorporated characteristics such as ranking, approval, and scoring \cite{rossi2011short}. A voter has a \textit{preference} between two candidates when it selects one over the other according to some criterion \cite{coombs1977single}. A voter can \textit{rank} all the candidates according to its preferences \cite{flach2012machine}. For example, \textit{Borda} assigns values to the $c$ candidates based on the voters' rankings: a voter's first choice receives a value of $c-1$, its second choice a value of $c-2$, and so on. The candidate with the largest total value is the winner. In other voting methods voters approve or disapprove each candidate rather than create a ranking. For example, in \textit{combined approval voting} (CAV) voters assign a score of -1, 0, or +1 to indicate disapproval, apathy, or approval, respectively. Lastly, some voting methods allow voters to indicate their level of approval with a score. For example, in \textit{range voting} each voter gives a score within a given range to each candidate, and the candidate with highest sum of scores wins.

\begin{figure*}[t]
\centering
\subfigure[Distance]{
\includegraphics[width = 0.49\linewidth]{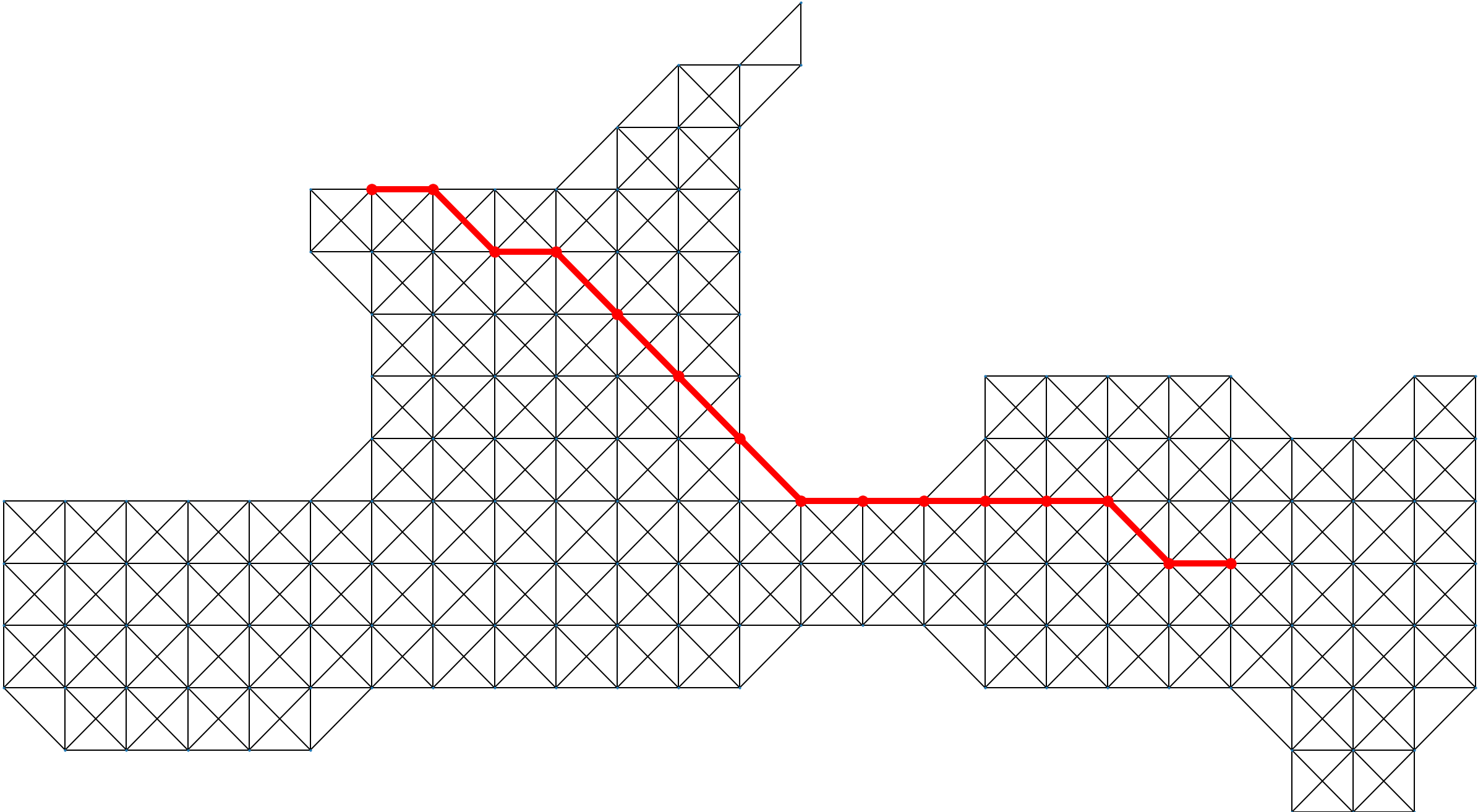}}
\subfigure[Uniform]{
\includegraphics[width = 0.49\linewidth]{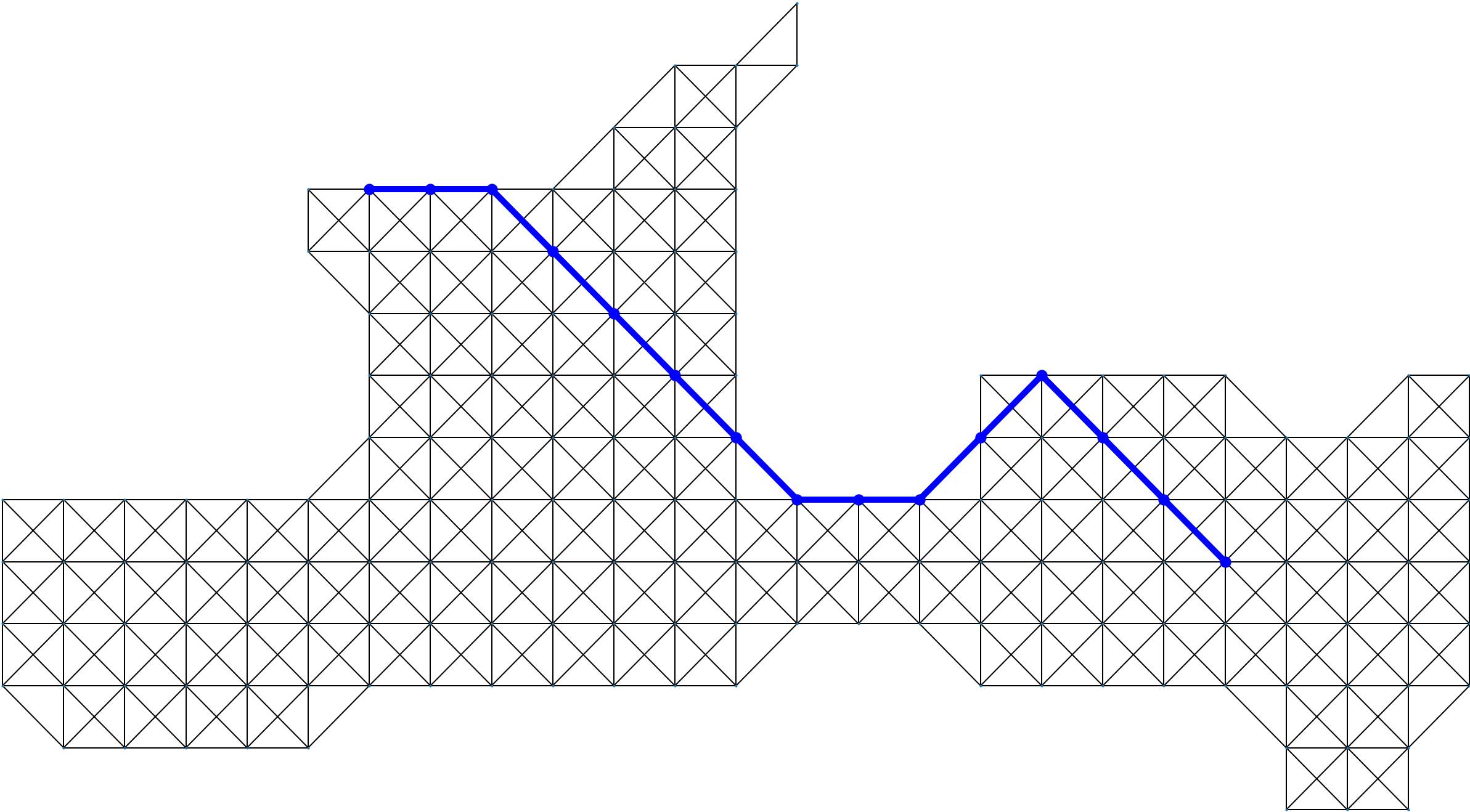}}
\subfigure[Safety]{
\includegraphics[width = 0.49\linewidth]{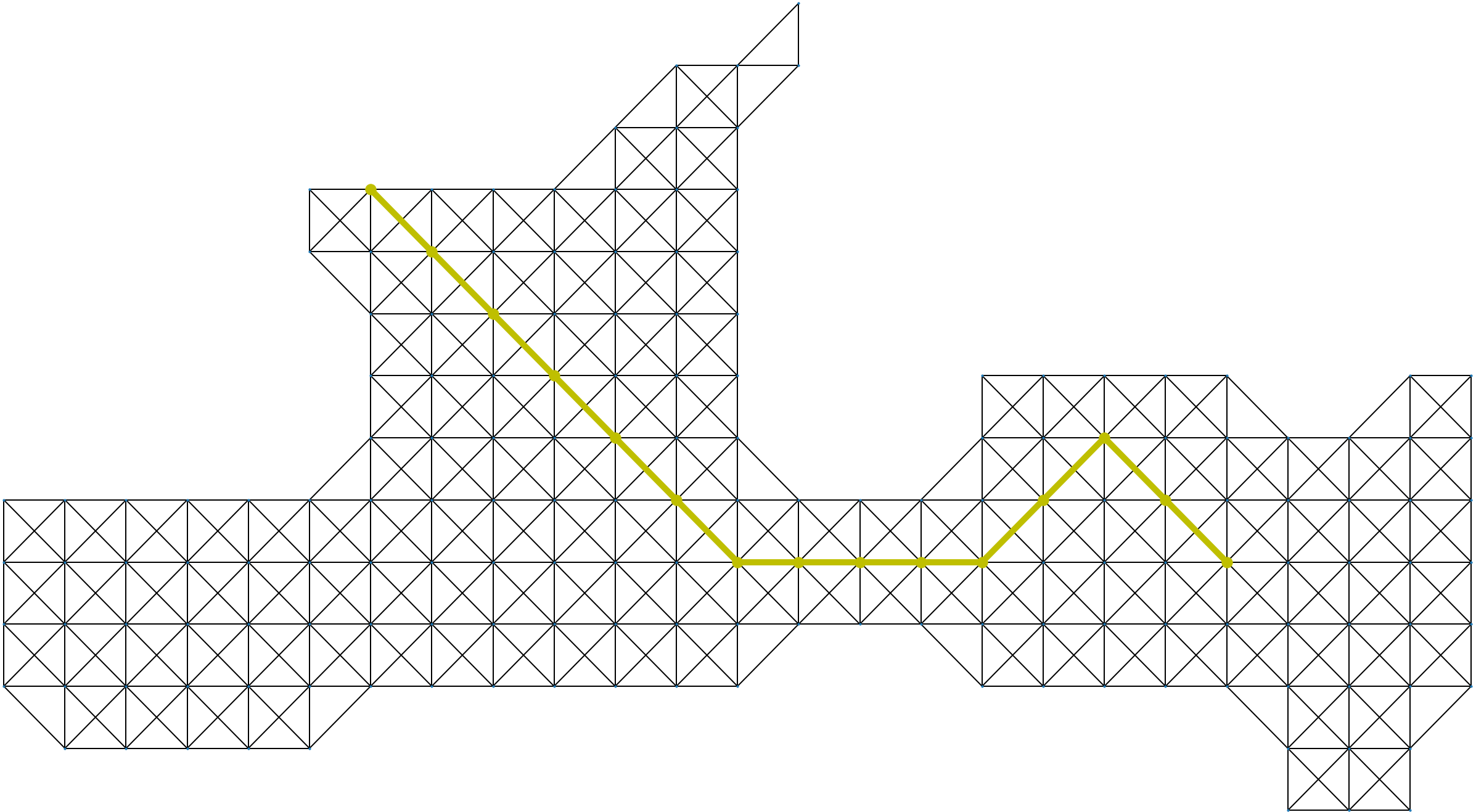}}
\subfigure[Random]{
\includegraphics[width = 0.49\linewidth]{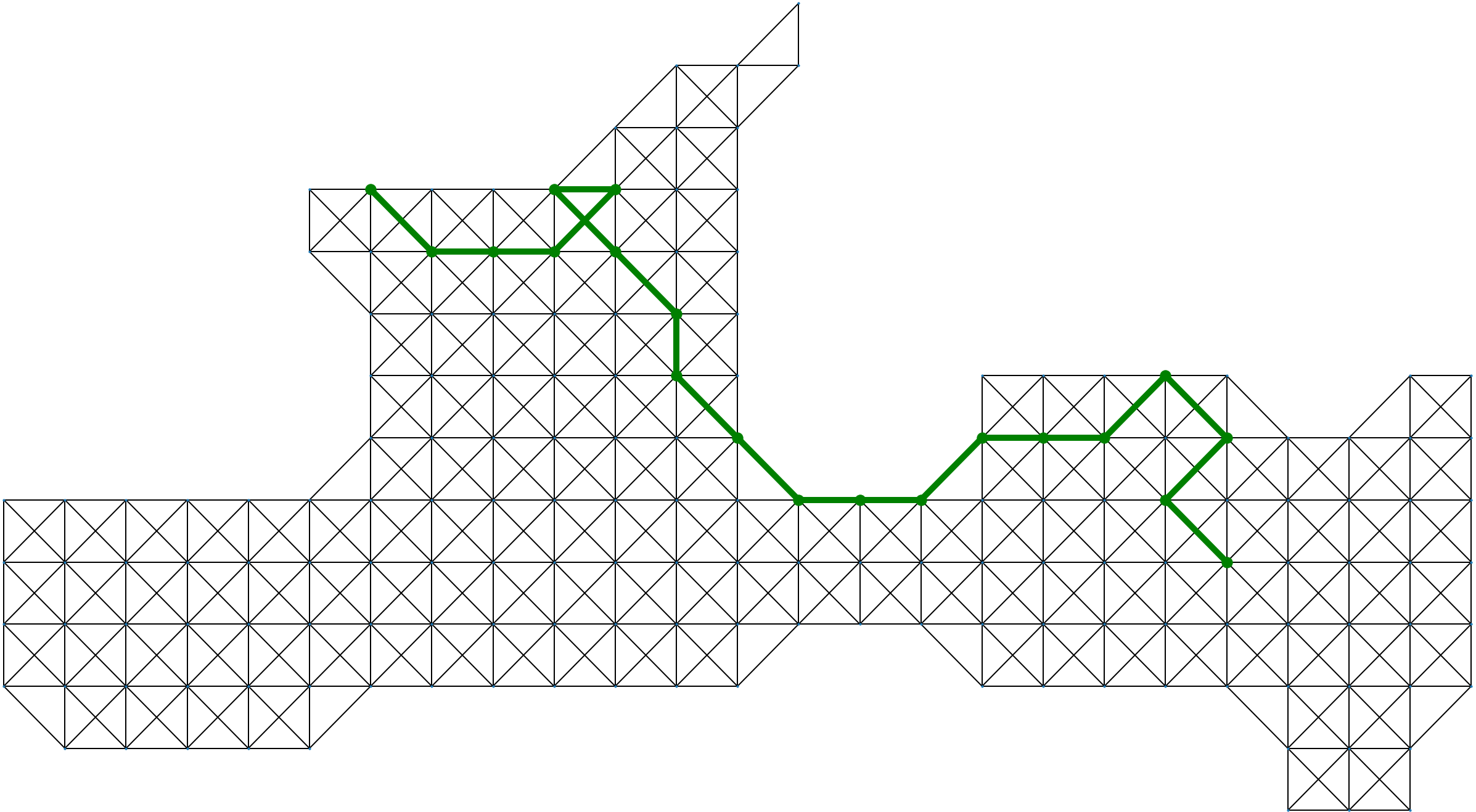}}
\caption{An example of four plans from VBMO in the environment of Figure \ref{fig:environment}. Each plan is optimal in its own objective. When VBMO used range voting, it selected the Safety-based plan because it has lowest total cost.}
\label{fig:example}
\end{figure*}

The approach closest to ours formulated multi-objective path planning as a reinforcement learning problem, and voted to select among the actions available at a state based on the expected reward under each objective \cite{tozer2017many}. It compared ranking, approval, and scoring voting methods. That approach, however, required hundreds of episodes of training and only considered an artificial $10 \times 20$ grid environment with four obstacles.

In previous work, we compared performance on a robot navigator's selected plan's objective with a fixed alternative objective to generate natural language explanations \cite{korpan2018toward}. That approach was then used on cognitively-based robot navigation system that learned three spatial models and compared plans that focused on each model to select and explain the plan that best exploited the robot's models \cite{epstein2019planning}. VBMO was first introduced as a way to generally provide contrastive natural language explanations of multi-objective path planning \cite{korpan2021contrastive}.

VBMO uses a topologically identical graph but with a set of edge labels that represents different objectives. It constructs an optimal plan for each objective, evaluates each plan with every objective, and then selects the plan with voting (i.e., range voting would select the plan with lowest total cost across all of them). This avoids the limitations of other approaches because it addresses each objective independently and then evaluates the resultant plans from the perspective of each planner. In \cite{korpan2021contrastive}, VBMO was only evaluated as a mechanism to produce natural language explanations. It was examined in a single environment with eight objectives, six of which relied on different aspects of a spatial model learned as the robot navigated through the environment. The result was a set of objectives that were highly correlated and context specific. Here, we generalize to many more environments with diverse configurations and complexities, with different types of objectives and voting mechanisms, and compare performance against a simple weighted sum of the objectives.

\section{Voting-based Multi-objective Path Planning}
VBMO constructs multiple plans, each of which optimizes a single objective, and then uses voting to select the plan that maximally satisfies the most objectives. It uses A* for its graph search algorithm. Pseudocode for VBMO appears in Algorithm \ref{alg:planning}. First, a graph is created to reflect the topology of the environment such that each vertex is unobstructed and edges are traversable. Figure \ref{fig:environment} shows an example of an environment along with its graph representation. Next, the edges are labeled with a cost vector that reflects the cost with respect to each objective. Figure \ref{fig:costs} shows an example of the edge costs based on three different objectives in Figure \ref{fig:environment}'s graph. Then VBMO constructs an optimal plan $P$ with respect to each objective in that graph. In this way, each plan is guaranteed to be optimal for at least one objective. Figure \ref{fig:example} provides an example of four plans generated by VBMO based on four objectives.

Once it assembles the set of plans $\mathcal{P}$, VBMO uses each objective to evaluate all of them. Because the underlying graph has the same topological structure (vertices and edges), every edge $e \in E$ in any plan is labeled by all the objectives. To evaluate planner $SO_i$'s plan $P_i=\langle v_1, v_2, \ldots, v_m \rangle$ from the perspective of objective $\beta_j$, VBMO sums $\beta_j$'s edge costs from the sequence of vertices in $P_i$. In this way, each objective $\beta_j$ calculates a \textit{score} $C_{ij}$ for each stored plan $P_i$. 

To avoid any biases that would be introduced by the magnitude of an objective's values, all scores from any $\beta_j$ are normalized in $[0,1]$. Because VBMO seeks to minimize its objectives, a score $C_{ij}$ near 0 indicates that plan $P_i$ closely conforms to objective $\beta_j$, while a score near 1 indicates that $P_i$ strongly violates $\beta_j$. Once every objective scores every plan, the plan $P_{best}$ is selected with voting.

VBMO has three voting mechanisms available to it. Range voting selects the plan with the lowest total score from all $J$ planners:
\begin{align} \label{eq:range}
P_{best}=\underset{P_i \in \mathcal{P}}{argmin} \ \sum_{j=1}^J C_{ij}
\end{align}
Borda voting first assigns a rank $r_{ij}$ to each plan's score $C_{ij}$ for each objective $\beta_j$ and then assigns points to each plan as $(J+1)-r_{ij}$. It selects the plan with maximum total points:
\begin{align} \label{eq:borda}
P_{best}=\underset{P_i \in \mathcal{P}}{argmax} \ \sum_{j=1}^J (J+1)-r_{ij}
\end{align}
Combined approval voting assigns values $v_{ij}$ to scores as
\[
    v_{ij}= 
\begin{cases}
    -1,& \text{if } C_{ij} = 1\\
    1,& \text{if } C_{ij} = 0\\
    0,              & \text{otherwise.}
\end{cases}
\]
It selects the plan with maximum total value:
\begin{align} \label{eq:cav}
P_{best}=\underset{P_i \in \mathcal{P}}{argmax} \ \sum_{j=1}^J v_{ij}
\end{align}
An example with all three voting methods appears in Table \ref{tab:votingexample}. It demonstrates that VBMO produces a plan that is optimal for each objective, and that voting methods can balance performance among all the objectives differently, which can change the plan that is selected. VBMO is guaranteed to construct at least one plan on the Pareto frontier (proof shown in \cite{korpan2021contrastive}). Furthermore, VBMO's voting mechanisms are guaranteed to always select a non-dominated plan.

\begin{theorem} Range voting always selects a non-dominated plan. \end{theorem}
\begin{proof}[Proof\nopunct] by contradiction. Assume range voting selected a plan $P_2$ that is dominated by plan $P_1$. By the definition of dominance this means that $\beta(P_1) \leq \beta(P_2)$ for every $\beta \in B$ and $\beta_j(P_1) < \beta_j(P_2)$ for at least one objective $\beta_j \in B$. Because range voting selects a plan with minimal score, we must minimize $P_2$'s score while also maintaining the dominance assumption. Suppose that $\beta(P_1) = \beta(P_2)$ for $\beta_1,...,\beta_{J-1}$ and is $\beta_J(P_1) < \beta_J(P_2)$ for only one objective $\beta_J$.

Range voting selects the plan with minimum total score $Score_i = \sum_{j=1}^J \beta_j(P_i)$. This is equivalent to $Score_i = (\sum_{j=1}^{J-1} \beta_j(P_i)) + \beta_J(P_i)$. When comparing $Score_1$ with $Score_2$, the first term $(\sum_{j=1}^{J-1} \beta_j(P_i))$ is equal in both, so the only difference is the second term $\beta_J(P_i)$. As shown earlier, $\beta_J(P_1) < \beta_J(P_2)$, so the total scores will be $Score_1 < Score_2$. 

For range voting to select $P_2$ over $P_1$, however, $Score_2$ must be less than $Score_1$ because it selects the plan with minimum score. Our assumption that range voting selected a plan $P_2$ that is dominated by plan $P_1$ must then be false and since dominance is transitive, no dominated plan could be selected, which proves the theorem must be true. \end{proof}

Similar proofs can trivially show that Borda voting and combined approval voting also always select a non-dominated plan.

In summary, VBMO is an efficient multi-objective path planning approach that always identifies and then selects a plan on the Pareto frontier, without reliance on finely-tuned weights. Theorem 1 proves that, given a set of dominated and non-dominated plans, VBMO's range voting mechanism will always select a non-dominated plan because non-dominated plans score lower with respect to at least one objective and therefore have a lower total score. Given $J$ objectives, VBMO has complexity $\mathcal{O}(J^2)$ because each plan is evaluated under each objective.

\section{Experimental Results}
\begin{table}[t]
\centering
\caption{Test environments}
\label{tab:environments}
\begin{tabular}{|p{0.81in}|l|p{0.5in}|p{0.5in}|p{0.5in}|}
\hline
Environment         & \# & Num. of Nodes & Num. of Edges & Average Degree \\ \hline
Dragon Age: Origins Grids & 156    & 168 to 137,375 & 552 to 530,551 & 6.5 to 7.9      \\ \hline
DIMACS NY Road Network & 1      & 365,050       & 264,346       & 2.8            \\ \hline
\end{tabular}
\end{table}

We compare VBMO with \textit{Weighted}, which uses A* search on a simple equally-weighted sum of the objectives. These were implemented in Python and evaluated on 156 Dragon Age: Origins (DAO) benchmark grid environments \cite{sturtevant2012benchmarks} and the NY road network from the 9th DIMACS Implementation Challenge: Shortest Path \footnote{http://www.diag.uniroma1.it/~challenge9/download.shtml}. Table \ref{tab:environments} summarizes details about these environments. We use Euclidean distance for A*'s heuristic in the DAO environments and Haversine distance for the DIMACS NY road network. Performance is measured by the average total score for the selected plan and the time to compute the plan. For each environment, performance is averaged over 50 randomly selected start and target vertices. All significant differences are at $p < 0.05$.

\begin{figure*}[t]
\centering
\subfigure[Scores]{
\includegraphics[width = 0.49\linewidth]{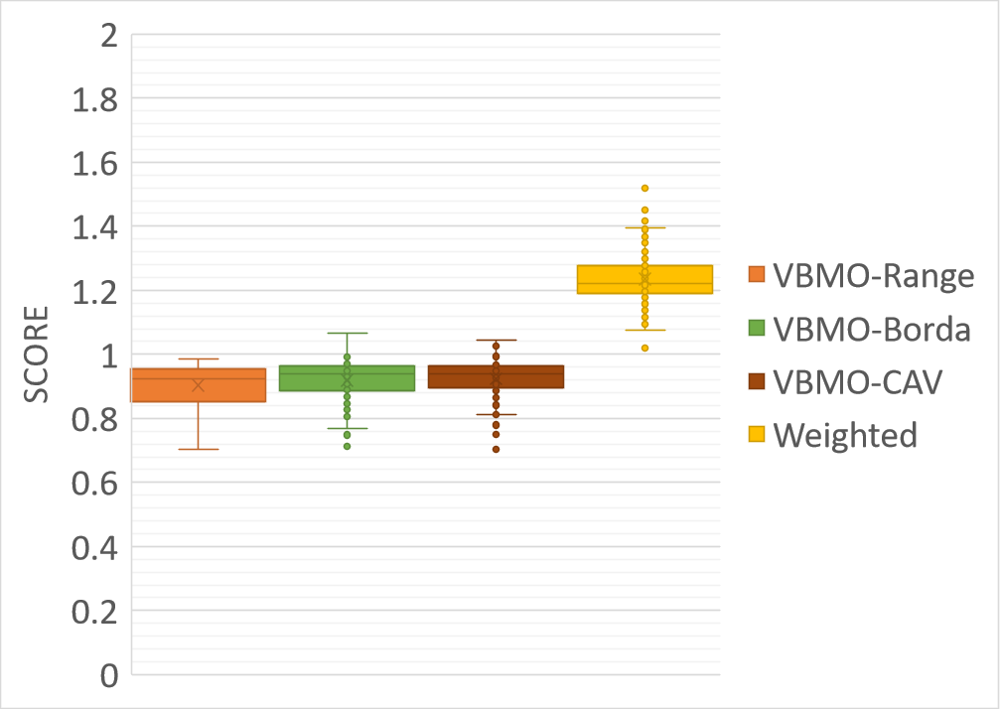}}
\subfigure[Times]{
\includegraphics[width = 0.49\linewidth]{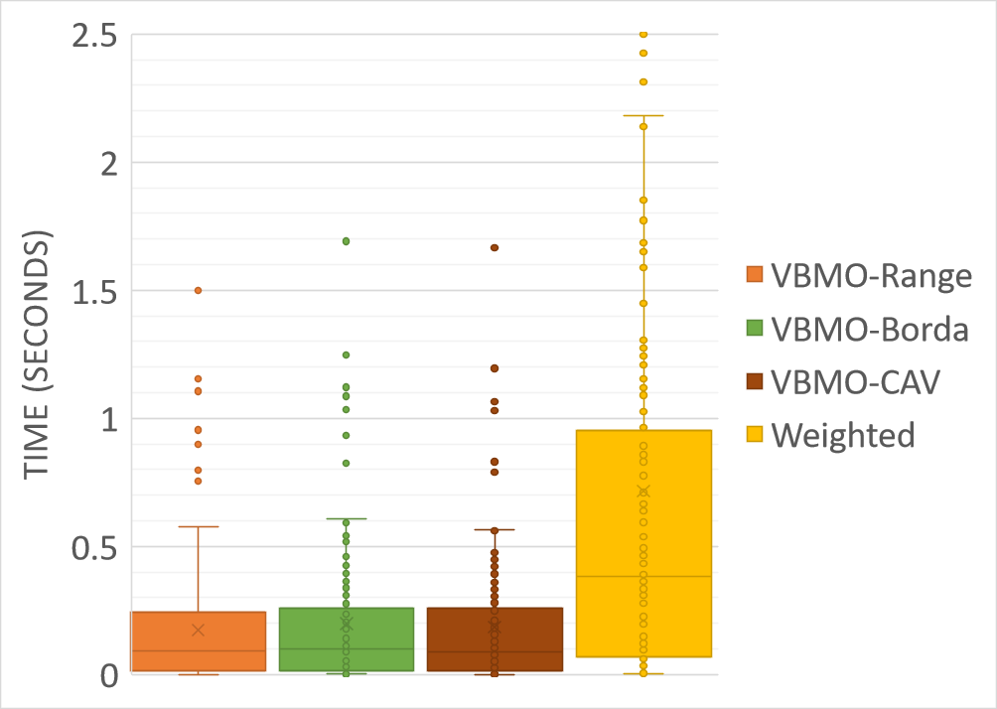}}
\caption{The distributions of the scores and times for the first experiment. (a) The average scores for VBMO-Range, VBMO-Borda, VBMO-CAV, and Weighted are 0.904, 0.920, 0.924, and 1.235, respectively. (b) The average times for VBMO-Range, VBMO-Borda, VBMO-CAV, and Weighted are 0.175s, 0.198s, 0.186s, and 0.717s, respectively. Both scores and times are significantly less than Weighted.}
\label{fig:exp1}
\end{figure*}

The first experiment compared performance on the DAO environments with three objectives: distance and random as described in Figure \ref{fig:costs}, and a uniform cost of 1 on all edges. These three objectives would minimize distance travelled, minimize number of steps taken, and minimize a random uniform value (e.g., traffic in an environment). With three objectives, the total score a plan can range from 0 to 2, with 0 indicating that the plan optimally adheres to all three objectives, and 2 indicating that the plan performed worst on the other two objectives. (A plan cannot get a score of 1 from its own objective because it optimally minimizes its own objective.) VBMO's score and time was significantly better than Weighted for all the environments and with all three voting mechanisms. Figure \ref{fig:exp1} compares the distributions of the scores and times. Table \ref{tab:exp1} shows how frequently each plan was selected by VBMO.

\begin{table}[t]
\centering
\caption{In the first experiment, averaged over the 156 environments among the 50 runs for each environment, voting selected the distance-based plan more often than the other two plans.}
\label{tab:exp1}
\begin{tabular}{|l|l|l|l|}
\hline
      & Distance & Uniform & Random \\ \hline
Range & 76.9\%   & 22.8\%  & 0.3\%  \\ \hline
Borda & 88.5\%   & 11.1\%  & 0.4\%  \\ \hline
CAV   & 88.6\%   & 11.1\%  & 0.3\%  \\ \hline
\end{tabular}
\end{table}

One challenge with using a weighted sum is to weight the objectives so that search is not biased toward any one objective because of the magnitude of its values. In the first experiment the the random objective could have a value much larger than the other two objectives, so the simple weighted sum would be biased toward plans that most satisfy that objective. VBMO does not have this problem because it normalizes the scores before calculating the total scores.

In the second experiment, the objectives were modified so that their values would be approximately equal. Distance remained the same, uniform was changed to 1.5 cost on all edges, and random would select a cost of 1 or 2 for each edge. In this case all three objectives are relatively similar in magnitude. The result is that VBMO with range voting performs significantly worse in terms of average score in 151 of the DAO environments, and no significant difference in 5 environments. The average score for VBMO with range voting across the 156 environments is 0.855 compared to 0.536 for Weighted. In this situation, it is possible VBMO performed worse because it is designed to find the plans at the extremes of the Pareto frontier, and Weighted can find a plan somewhere in the middle, which could end up having a smaller total score because it does equally well on all three objectives without being optimal in any single one. Despite having worse scores, VBMO was still significantly faster than Weighted with an average time of 0.396s compared to 0.831s.

\begin{table*}[t]
\centering
\caption{Comparison of performance with different objectives. Asterisk indicates significant difference between VBMO and Weighted.}
\label{tab:dimacs}
\begin{tabular}{|l|rr|rr|}
\hline
                                                    & \multicolumn{2}{c|}{Score}       & \multicolumn{2}{c|}{Time}        \\
Objectives Used & \multicolumn{1}{c|}{VBMO} & \multicolumn{1}{c|}{Weighted} & \multicolumn{1}{c|}{VBMO} & \multicolumn{1}{c|}{Weighted} \\ \hline
Distance, Time, Uniform = 10                        & \multicolumn{1}{r|}{1.12} & 1.05 & \multicolumn{1}{r|}{3.24*} & 3.68 \\ \hline
Distance, Time, Uniform = 2000                      & \multicolumn{1}{r|}{1.05} & 0.78* & \multicolumn{1}{r|}{3.07*} & 3.56 \\ \hline
Distance, Time, Uniform = 20, Random =   {[}1,20{]} & \multicolumn{1}{r|}{1.27*} & 1.43 & \multicolumn{1}{r|}{4.64} & 4.33 \\ \hline
Distance, Uniform = 20, Random = {[}1, 20{]}        & \multicolumn{1}{r|}{0.93*} & 1.81 & \multicolumn{1}{r|}{3.02*} & 3.52 \\ \hline
Distance, Uniform = 1000, Random = {[}1,   37000{]} & \multicolumn{1}{r|}{0.92*} & 1.04 & \multicolumn{1}{r|}{2.81*} & 3.34 \\ \hline
\end{tabular}
\end{table*}

The first two experiments used a random objective, which introduces uncertainty. To examine if uncertainty had an effect on the performance, the third experiment uses distance, uniform cost of 1.5, and the safety objective described in Figure \ref{fig:costs}. The results showed that of the 156 DAO environments, VBMO with range voting scored significantly better in 118, significantly worse in 2, and no different in 36. In the 118 environments with better scores VBMO scored 0.469 compared to 0.803 for Weighted on average. For the two worse environments, VBMO scored 0.328 compared with 0.18 for Weighted. Finally, for the 36 where there was no statistically significant difference, VBMO scored lower on average (0.463) than Weighted (0.539). An examination of the environment's structures, graph density, and average distance to the target did not reveal a readily apparent reason for the environments where VBMO performed better or worse. Current evaluation seeks to identify the environment's characteristics that affect VBMO's performance. VBMO was significantly faster than Weighted for all these environments, however, with an average time of 0.121s compared to 0.732s.

The fourth experiment examined performance in a much larger real-world environment, the DIMACS NY road network. This data comes with two objectives already in the data set: distance and time. Table \ref{tab:dimacs} shows the performance of VBMO with range voting on several configurations of objectives in this environment. The first two instances used distance, time, and a uniform cost. VBMO's score was no different when a small uniform cost was used, however, it was significantly worse with a larger uniform cost that was close to the average value of the distance and time objectives, similar to the result of the second experiment. The third instance added a random objective, which resulted in VBMO having a significantly better score. Given that distance and time are 96\% correlated in this data set, the last two instances only used distance along with uniform and random costs, both at different scales. In both of these cases, VBMO had a significantly better score. In terms of time, VBMO was significantly faster with three objectives, and no different with four objectives.

\section{Discussion}
This paper describes a flexible, scalable approach for multi-objective path planning, VBMO, that uses voting to select among single-objective planners. VBMO evaluates each plan with respect to each of the objectives. The plan that VBMO selects is optimal in its own objective and is guaranteed to be non-dominated.

Both evolutionary methods and VBMO consider a population of solutions and select among them with a kind of fitness function. VBMO, however, does not require multiple iterations to refine its plan. Instead, it starts with at least one plan already on the Pareto frontier and selects a plan that is generally expected to perform well with respect to all the objectives. Future work could consider a hybrid approach, that starts with VBMO's plans and then uses an evolutionary method to merge and transform those plans. In this way the best segments of different plans could be combined or problematic portions eliminated.

Here, Borda voting and combined approval voting were added to VBMO. Future work could consider additional voting methods for VBMO, such as a Condorcet method, to select a plan because those methods may have properties that more fairly and efficiently balance competing objectives to select among the plans. Although VBMO does not need hand-tuned weights to balance multiple objectives, they could be easily incorporated into any of the voting mechanisms to change an objective's influence on the sum. 

VBMO uses A* for its graph search algorithm but another optimal graph search algorithm could easily be substituted for it. Future work could examine how other search algorithms would affect VBMO's performance. In several of the experiments conducted here VBMO did not achieve a significantly better score. A potential reason for this could be that Euclidean distance was used as the heuristic for A* across all objectives in the DAO environments. Euclidean distance may not, however, be admissible for some objectives, such as random and uniform, so it may overestimate the cost to the goal. Current work considers alternative heuristics to ensure A* is optimal.

VBMO can encounter a task where its set of objectives $B$ generate plans that are equally poor on all the other objectives. In that case, total scores for all plans are equal and VBMO selects a plan at random, one that should perform well only on its own objective and poorly on the others. That reduces the solution to a planner with a single, randomly chosen objective. Other multi-objective planning methods avoid this difficulty by compromise among all the objectives rather than focus on strong performance from one. To address this issue, VBMO could incorporate additional planners that introduce weighted sums of different objectives so that the planner is forced to find a plan that compromises between them but would do less well on any single objective. Unless these additions were simple, it would become a problem of finding the best set of weights for the objectives (i.e., searching among the infinitely many points on the Pareto frontier).

As shown in our previous work, VBMO easily generates contrastive explanations in natural language. These explanations flexibly compare plans with respect to the objectives under consideration and express the controller's confidence in its selected plan. Current work considers how to use a language model to shorten, simplify, and produce more human-like explanations.

Although VBMO is applied here to path planning, it is more generally applicable to any multi-objective problem where a solution based on a single objective can be evaluated from the perspective of the other objectives. VBMO also only considers a single agent, future work could extend VBMO to address multi-agent multi-objective path planning.

VBMO is an efficient multi-objective path planning approach that generates plans on the Pareto frontier. The plan it selects with voting is guaranteed to be optimal with respect to at least one objective and likely does well on the others. The results with voting-based multi-objective path planning presented here demonstrate that this planning algorithm is fast, flexible, and scalable.

\appendix



\section*{Acknowledgments}
The author thanks Dr. Susan L. Epstein for her mentorship on the development of VBMO.

\bibliographystyle{ijcai23}
\bibliography{ijcai23}

\begin{thebibliography}{}

\bibitem[\protect\citeauthoryear{Aine \bgroup \em et al.\egroup
  }{2016}]{aine2016multi}
Sandip Aine, Siddharth Swaminathan, Venkatraman Narayanan, Victor Hwang, and
  Maxim Likhachev.
\newblock Multi-heuristic {A*}.
\newblock {\em The International Journal of Robotics Research},
  35(1-3):224--243, 2016.

\bibitem[\protect\citeauthoryear{Coombs and Avrunin}{1977}]{coombs1977single}
Clyde~H Coombs and George~S Avrunin.
\newblock Single-peaked functions and the theory of preference.
\newblock {\em Psychological review}, 84(2):216--230, 1977.

\bibitem[\protect\citeauthoryear{Deb \bgroup \em et al.\egroup
  }{2002}]{deb2002fast}
Kalyanmoy Deb, Amrit Pratap, Sameer Agarwal, and TAMT Meyarivan.
\newblock A fast and elitist multiobjective genetic algorithm: Nsga-ii.
\newblock {\em IEEE transactions on evolutionary computation}, 6(2):182--197,
  2002.

\bibitem[\protect\citeauthoryear{Epstein and
  Korpan}{2019}]{epstein2019planning}
Susan~L. Epstein and Raj Korpan.
\newblock Planning and explanations with a learned spatial model.
\newblock In Sabine Timpf, Christoph Schlieder, Markus Kattenbeck, Bernd
  Ludwig, and Kathleen Stewart, editors, {\em COSIT-2019}, volume 142 of {\em
  LIPICS}, pages 22:1--22:20. Schloss Dagstuhl, 2019.

\bibitem[\protect\citeauthoryear{Farina and Amato}{2002}]{farina2002optimal}
Marco Farina and Paolo Amato.
\newblock On the optimal solution definition for many-criteria optimization
  problems.
\newblock In {\em Proceedings of the North American Fuzzy Information
  Processing Society}, pages 233--238. IEEE, 2002.

\bibitem[\protect\citeauthoryear{Flach}{2012}]{flach2012machine}
Peter Flach.
\newblock {\em Machine learning: the art and science of algorithms that make
  sense of data}.
\newblock Cambridge University Press, 2012.

\bibitem[\protect\citeauthoryear{Gei{\ss}er \bgroup \em et al.\egroup
  }{2022}]{geisser2022admissible}
Florian Gei{\ss}er, Patrik Haslum, Sylvie Thi{\'e}baux, and Felipe Trevizan.
\newblock Admissible heuristics for multi-objective planning.
\newblock In {\em Proceedings of the International Conference on Automated
  Planning and Scheduling}, volume~32, pages 100--109, 2022.

\bibitem[\protect\citeauthoryear{Goldin and
  Salzman}{2021}]{goldin2021approximate}
Boris Goldin and Oren Salzman.
\newblock Approximate bi-criteria search by efficient representation of subsets
  of the pareto-optimal frontier.
\newblock In {\em Proceedings of the International Conference on Automated
  Planning and Scheduling}, volume~31, pages 149--158, 2021.

\bibitem[\protect\citeauthoryear{Haimes}{1973}]{haimes1973integrated}
Yacov~Y Haimes.
\newblock Integrated system identification and optimization.
\newblock In {\em Control and Dynamic Systems}, volume~10, pages 435--518.
  Elsevier, 1973.

\bibitem[\protect\citeauthoryear{Hart \bgroup \em et al.\egroup
  }{1968}]{hart1968formal}
P.~E. Hart, N.~J. Nilsson, and B.~Raphael.
\newblock A formal basis for the heuristic determination of minimum cost paths.
\newblock {\em IEEE Transactions on Systems Science and Cybernetics},
  4(2):100--107, July 1968.

\bibitem[\protect\citeauthoryear{Jaimes and Coello}{2015}]{jaimes2015many}
Antonio~L{\'o}pez Jaimes and Carlos A~Coello Coello.
\newblock Many-objective problems: challenges and methods.
\newblock In {\em Springer handbook of computational intelligence}, pages
  1033--1046. Springer, 2015.

\bibitem[\protect\citeauthoryear{Keeney \bgroup \em et al.\egroup
  }{1993}]{keeney1993decisions}
Ralph~L Keeney, Howard Raiffa, and Richard~F Meyer.
\newblock {\em Decisions with multiple objectives: preferences and value
  trade-offs}.
\newblock Cambridge University Press, 1993.

\bibitem[\protect\citeauthoryear{Korpan and Epstein}{2018}]{korpan2018toward}
Raj Korpan and Susan~L Epstein.
\newblock Toward natural explanations for a robot's navigation plans.
\newblock In {\em Proceedings of Workshop on Explainable Robotic Systems at HRI
  2018}, 2018.

\bibitem[\protect\citeauthoryear{Korpan and
  Epstein}{2021}]{korpan2021contrastive}
Raj Korpan and Susan Epstein.
\newblock Contrastive natural language explanations for multi-objective path
  planning.
\newblock In {\em ICAPS 2021 Workshop on Explainable AI Planning}, 2021.

\bibitem[\protect\citeauthoryear{LaValle}{2006}]{lavalle2006planning}
Steven~M LaValle.
\newblock {\em Planning algorithms}.
\newblock Cambridge University Press, 2006.

\bibitem[\protect\citeauthoryear{Lee and others}{1972}]{lee1972goal}
Sang~M Lee et~al.
\newblock {\em Goal programming for decision analysis}.
\newblock Auerbach Publishers Philadelphia, 1972.

\bibitem[\protect\citeauthoryear{Mandow and
  De~La~Cruz}{2008}]{mandow2008multiobjective}
Lawrence Mandow and Jos{\'e} Luis~P{\'e}rez De~La~Cruz.
\newblock Multiobjective {A}* search with consistent heuristics.
\newblock {\em Journal of the ACM (JACM)}, 57(5):1--25, 2008.

\bibitem[\protect\citeauthoryear{Marler and Arora}{2010}]{marler2010weighted}
R~Timothy Marler and Jasbir~S Arora.
\newblock The weighted sum method for multi-objective optimization: new
  insights.
\newblock {\em Structural and multidisciplinary optimization}, 41(6):853--862,
  2010.

\bibitem[\protect\citeauthoryear{Pardalos \bgroup \em et al.\egroup
  }{2008}]{pardalos2008pareto}
Panos~M Pardalos, Athanasios Migdalas, and Leonidas Pitsoulis.
\newblock {\em Pareto optimality, game theory and equilibria}, volume~17.
\newblock Springer Science \& Business Media, 2008.

\bibitem[\protect\citeauthoryear{Phillips \bgroup \em et al.\egroup
  }{2015}]{phillips2015efficient}
Mike Phillips, Venkatraman Narayanan, Sandip Aine, and Maxim Likhachev.
\newblock Efficient search with an ensemble of heuristics.
\newblock In {\em Twenty-Fourth International Joint Conference on Artificial
  Intelligence}, 2015.

\bibitem[\protect\citeauthoryear{Refanidis and
  Vlahavas}{2003}]{refanidis2003multiobjective}
Ioannis Refanidis and Ioannis Vlahavas.
\newblock Multiobjective heuristic state-space planning.
\newblock {\em Artificial Intelligence}, 145(1-2):1--32, 2003.

\bibitem[\protect\citeauthoryear{Ren \bgroup \em et al.\egroup
  }{2021a}]{ren2021multi}
Zhongqiang Ren, Sivakumar Rathinam, and Howie Choset.
\newblock Multi-objective conflict-based search for multi-agent path finding.
\newblock In {\em 2021 IEEE International Conference on Robotics and Automation
  (ICRA)}, pages 8786--8791. IEEE, 2021.

\bibitem[\protect\citeauthoryear{Ren \bgroup \em et al.\egroup
  }{2021b}]{ren2021subdimensional}
Zhongqiang Ren, Sivakumar Rathinam, and Howie Choset.
\newblock Subdimensional expansion for multi-objective multi-agent path
  finding.
\newblock {\em IEEE Robotics and Automation Letters}, 6(4):7153--7160, 2021.

\bibitem[\protect\citeauthoryear{Rossi \bgroup \em et al.\egroup
  }{2011}]{rossi2011short}
Francesca Rossi, Kristen~Brent Venable, and Toby Walsh.
\newblock A short introduction to preferences: between artificial intelligence
  and social choice.
\newblock {\em Synthesis Lectures on Artificial Intelligence and Machine
  Learning}, 5(4):1--102, 2011.

\bibitem[\protect\citeauthoryear{Sedeno-Noda and
  Raith}{2015}]{sedeno2015dijkstra}
Antonio Sedeno-Noda and Andrea Raith.
\newblock A dijkstra-like method computing all extreme supported non-dominated
  solutions of the biobjective shortest path problem.
\newblock {\em Computers \& Operations Research}, 57:83--94, 2015.

\bibitem[\protect\citeauthoryear{Stewart and
  White~III}{1991}]{stewart1991multiobjective}
Bradley~S Stewart and Chelsea~C White~III.
\newblock Multiobjective {A*}.
\newblock {\em Journal of the ACM (JACM)}, 38(4):775--814, 1991.

\bibitem[\protect\citeauthoryear{Sturtevant}{2012}]{sturtevant2012benchmarks}
N.~Sturtevant.
\newblock Benchmarks for grid-based pathfinding.
\newblock {\em Transactions on Computational Intelligence and AI in Games},
  4(2):144 -- 148, 2012.

\bibitem[\protect\citeauthoryear{Talbi \bgroup \em et al.\egroup
  }{2012}]{talbi2012multi}
El-Ghazali Talbi, Matthieu Basseur, Antonio~J Nebro, and Enrique Alba.
\newblock Multi-objective optimization using metaheuristics: non-standard
  algorithms.
\newblock {\em International Transactions in Operational Research},
  19(1-2):283--305, 2012.

\bibitem[\protect\citeauthoryear{Tozer \bgroup \em et al.\egroup
  }{2017}]{tozer2017many}
Bentz Tozer, Thomas Mazzuchi, and Shahram Sarkani.
\newblock Many-objective stochastic path finding using reinforcement learning.
\newblock {\em Expert Systems with Applications}, 72:371--382, 2017.

\bibitem[\protect\citeauthoryear{Ulloa \bgroup \em et al.\egroup
  }{2020}]{ulloa2020simple}
Carlos~Hern{\'a}ndez Ulloa, William Yeoh, Jorge~A Baier, Han Zhang, Luis Suazo,
  and Sven Koenig.
\newblock A simple and fast bi-objective search algorithm.
\newblock In {\em Proceedings of the International Conference on Automated
  Planning and Scheduling}, volume~30, pages 143--151, 2020.

\bibitem[\protect\citeauthoryear{Van~Erp \bgroup \em et al.\egroup
  }{2002}]{van2002overview}
Merijn Van~Erp, Louis Vuurpijl, and Lambert Schomaker.
\newblock An overview and comparison of voting methods for pattern recognition.
\newblock In {\em Proceedings of the Eighth International Workshop on Frontiers
  in Handwriting Recognition}, pages 195--200. IEEE, 2002.

\bibitem[\protect\citeauthoryear{Wierzbicki}{1980}]{wierzbicki1980use}
Andrzej~P Wierzbicki.
\newblock The use of reference objectives in multiobjective optimization.
\newblock In {\em Multiple criteria decision making theory and application},
  pages 468--486. Springer, 1980.

\bibitem[\protect\citeauthoryear{Zadeh}{1963}]{zadeh1963optimality}
Lofti Zadeh.
\newblock Optimality and non-scalar-valued performance criteria.
\newblock {\em IEEE transactions on Automatic Control}, 8(1):59--60, 1963.

\bibitem[\protect\citeauthoryear{Zhang \bgroup \em et al.\egroup
  }{2022}]{zhang2022pex}
Han Zhang, Oren Salzman, TK~Satish Kumar, Ariel Felner, Carlos~Hern{\'a}ndez
  Ulloa, and Sven Koenig.
\newblock A* pex: Efficient approximate multi-objective search on graphs.
\newblock In {\em Proceedings of the International Conference on Automated
  Planning and Scheduling}, volume~32, pages 394--403, 2022.

\end{thebibliography}

\end{document}